\documentclass[a4paper, 11pt]{article}

\usepackage[margin=2cm]{geometry}
\usepackage{amsmath}
\usepackage{amssymb}
\usepackage{amsthm}
\usepackage{bm} 
\usepackage{amsfonts}
\usepackage{mathtools}
\usepackage{graphicx}
\usepackage{parskip}
\usepackage{tikz}
\usepackage{standalone}
\usepackage[most]{tcolorbox}
\usepackage[colorlinks, linkcolor=black, citecolor=black]{hyperref}
\usepackage{todonotes}
\usepackage{caption}
\usepackage{subcaption}
\usepackage{authblk} 

\newcommand{\R}{\mathbb{R}}
\newcommand{\hmu}{\hat{\mu}_{n}}

\newcommand{\Pw}{\mathcal{P}_{2}\left( \R^{d+2} \right) }
\newcommand{\PPw}{\mathcal{P}_{2}\left(X\right)}

\newcommand{\btheta}{\bm{\theta}}
\newcommand\scalemath[2]{\scalebox{#1}{\mbox{\ensuremath{\displaystyle #2}}}}
\DeclareMathOperator*{\argmin}{argmin}

\theoremstyle{plain} 
\newtheorem{proposition}{Proposition}

\graphicspath{figures}

\title{Mean-Field Model for Two-Layer Neural Networks Trained with Consensus-Based Optimization}
\author[1,2]{William De Deyn}
\author[1,3]{Michael Herty}
\author[2]{Giovanni Samaey}
\affil[1]{Institut f{\"u}r Geometrie und Praktische Mathematik, RWTH Aachen University, Germany}
\affil[2]{Department of Computer Science, KU Leuven, Belgium}
\affil[3]{Extraordinary Professor, Department of Mathematics and Applied Mathematics, University of Pretoria, South Africa}

\date{\today}

\begin{document}

	\maketitle

	\begin{abstract}
	\noindent
	We study Consensus-Based Optimization (CBO) for two-layer neural network training. We compare the performance of CBO against Adam on two test cases and demonstrate how a hybrid approach, combining CBO with Adam, provides faster convergence than CBO. Additionally, in the context of multi-task learning, we recast CBO into a formulation that offers less memory overhead. The CBO method allows for a mean-field model formulation, which we couple with the mean-field model of the neural network. To this end, we first reformulate CBO within the optimal transport framework. As the number of particles tends to infinity, we lift the corresponding dynamics to the Wasserstein-over-Wasserstein space and show that the variance decreases monotonically. We confirm numerically that both mean-field models converge.
	\end{abstract}

	\section{Introduction}
	Artificial Intelligence has witnessed remarkable progress over the past decades, both in its capabilities and its range of applications. Today, neural networks are present in a variety of fields. One classical application is function approximation, which is supported by the universal approximation theory~\cite{hornik1989}. In computer vision, convolutional neural networks form the backbone of most modern architectures~\cite{lecun2015,krizhevsky2017}, while the framework of neural ordinary differential equations has contributed significantly to optimal control problems~\cite{chen2018,bottcher2022}. In natural language processing and speech recognition, recurrent neural networks and the long short-term memory variants have yielded significant performance improvements~\cite{hochreiter1997,rumelhart1986}. More recently, diffusion models have illustrated to be powerful generative models, with applications ranging from image denoising to video generation~\cite{yang2024}. Neural networks have even found their way into scientific computing. The most notable example is physics-informed neural networks, which are capable of solving both forward and inverse problems governed by partial differential equations~\cite{raissi2019}. 

	A neural network can be viewed, in general, as a function parametrized by a set of weights and biases, which we collectively refer to as parameters. A two-layer neural network, for example, is written as
	\begin{equation}\label{intro:nn_matrix}
		\hat{g}(\bm{x}; \bm{\theta}) = \bm{c}^{\top} \sigma \left( W\bm{x} + \bm{b} \right), 
	\end{equation}
	with $ \bm{x} \in \R^{d}, W \in \R^{M \times d}, \bm{b}, \bm{c} \in \R^{M}$ and $ \sigma(x) : \R \to \R $ the activation function. Here, \mbox{$\bm{\theta}  = \left(W, b, c\right) \in \R^{d_{o}}$} represents a vector containing all parameters of the neural network, with $ d_{o} = M(d+2) $ denoting the optimization dimension. A key factor behind the success of neural networks across diverse applications is their ability to approximate highly complex functions by suitable choice of $ \bm{\theta} $. The process of finding the best parameters $ \bm{\theta} $ is more commonly known as training the neural network. Training is typically formulated as the minimization of the empirical risk function
	\begin{equation}\label{intro:emp_risk}
		\hat{R}(\bm{\theta}) = \frac{1}{S}\sum_{s=1}^{S} \mathcal{L}\left(y_{s}, \hat{g} \left( \bm{x}_s; \bm{\theta} \right)\right), 
	\end{equation}
	where $ \mathcal{L} $ represents a loss function and $ S $ the number of data points~\cite{vapnik1999}. The loss function denotes the discrepancy between a true sample $y_{s}$ and the model prediction made from the input sample $\bm{x}_{s}$. Currently, the standard method to find the minimizer of \eqref{intro:emp_risk} is stochastic gradient descent (SGD) or adaptive variants, such as Adam~\cite{bottou-98x,kingma2015adam}. The optimization of the risk function is a difficult task~\cite{goodfellowDeepLearning2016,glorot2010}, primarily because the objective function is in general nonconvex. Gradient-based methods, such as SGD, are susceptible to becoming trapped in local minima. Other well-known difficulties include the vanishing gradient and exploding gradient phenomena~\cite{pascanu2013,bengio1994}. One alternative approach is to apply particle-based methods, such as Consensus-Based Optimization (CBO). The CBO method is a gradient free, global optimization method designed for high-dimensional, nonconvex objective functions~\cite{pinnauConsensusbasedModelGlobal2017a, carrillo2018, carrillo2021, gerber2025, carrillo2022b, huang2022, kossMeanFieldLimit2024a, fornasier2022, fornasier2024, fornasier2021, borghi2023, carrillo2023, borghi2022,ha2020, ha2021}. The CBO method allows for the passage to the mean-field limit as the number of particles tends to infinity.

	In this paper, we study the feasibility of training neural networks using \mbox{Consensus-Based Optimization} and compare its performance against the popular Adam method. We propose a hybrid method, combining CBO and Adam, and illustrate that the hybrid method has fast and stable convergence. Owing to the natural mean-field formulation of CBO, we also study two mean-field models in this work. We first consider the mean-field model of two-layer neural networks, similar to the works~\cite{mei2018,nguyen2023,sirignano2020,sirignano2020a,bach2022, Chizat2022MeanField}. Next, in the same spirit, we derive a time-discrete mean-field model governing the CBO dynamics on infinitely wide neural networks. The mean-field model is useful both conceptually and methodologically. Conceptually, it provides a probabilistic representation of training in the infinite-width regime by casting the problem as optimization over measures. Methodologically, this framework enables a rigorous analysis of particle approximations: we show, for example, that the variance of the particle population decreases along iterations and converges to zero under mild regularity and optimality assumptions.

	The remainder of this paper is organized as follows. Section \ref{section2} reviews the Barron space, which is a function space containing functions that can be efficiently represented by two-layer neural networks. Next, in Section \ref{section3}, we review the general training problem in machine learning and shortly discuss Adam and Consensus-Based Optimization. We also present two variants of CBO, such as the hybrid method and Multi-Task CBO. In Section \ref{section4}, we cover numerical experiments comparing the performance of CBO to Adam in training neural networks. Additionally, we provide an experiment illustrating the capabilities of Multi-Task CBO. Section \ref{section5} establishes the mean-field models in two regimes: first, as the number of hidden neurons tends to infinity; and second as the number of optimization particles tends to infinity. In the latter regime, we prove that the variance of the particle ensemble decreases 
	monotonically. We confirm both mean-field models numerically by showing that the empirical risk decreases monotonically in the network width and in the number of particles. Finally, Section \ref{section6} summarizes the main findings and presents the conclusion.
	
	\section{The Barron Space}\label{section2}
	
	In this section, we introduce the Barron space, which is a function space suited for approximation by two-layer neural networks. First, we rewrite and scale the neural network formulation in \eqref{intro:nn_matrix} by considering the two-layer neural network as a function $\hat{g}_{M}(\bm{x};\theta): X \to \mathbb{R}$ of the form
	\begin{equation}\label{intro:disc_nn}
	\hat{g}_{M}(\bm{x};\bm{\theta}) = \frac{1}{M}\sum_{m=1}^M c_m\sigma(\bm{w}_m^\top \bm{x} + b_m),
	\end{equation}
	with $\bm{x} \in X \subseteq \R^{d}, \bm{w}_{m} \in \R^{d}, b_m, c_m \in \R$~\cite{eBarronSpaceFlowInduced2022}. The parameter $M \in \mathbb{N}$ denotes the number of neurons in the hidden layer, i.e., the width of the two-layer neural network. The neural network is parameterized by weights $w_{m, j}, c_{m}$ and biases $b_{m}$, which we collect into the parameter vector \mbox{$\bm{\theta} \coloneq \{w_{m,j}, b_m, c_m \}_{m=1}^{M},~j=1, \dots, d$}. For an empirical measure $\hat{\mu} \in \mathcal{P}_{2}(\R^{d+2})$ defined by
	\begin{equation}
		\hat{\mu}(\bm{w}, b, c) = \frac{1}{M}\sum_{m=1}^{M} \delta(\bm{w} - \bm{w}_m)\delta(b - b_m)\delta(c - c_m), 
	\end{equation}
	the neural network in \eqref{intro:disc_nn} can be rewritten as a function $ \hat{g}_M(\bm{x}): X \to \R$ 
	\begin{equation}\label{intro:integral_disc_nn}
	\hat{g}_{M}(\bm{x}) = \int_{\Omega} c \sigma(\bm{w}^\top \bm{x} + b) d\hat{\mu}(\bm{w}, b, c), \qquad \Omega = \R^d \times \R \times \R,
	\end{equation}
	and where $\Omega$ represents our parameter space. We restrict the representing measures to $ \Pw$, following~\cite{chizat2020implicit}. In the mean-field regime of neural networks, we consider the limit $M \to \infty$. For an arbitrary probability measure $ \mu \in \Pw $, the mean-field limit yields a representation of a neural network as
	\begin{equation}\label{intro:cont_nn}
			g(\bm{x}) = \int_{\Omega} c \sigma(\bm{w}^\top \bm{x} + b) d\mu(\bm{w}, b, c) = \mathbb{E}_{\mu}\left[ c \sigma(\bm{w}^\top \bm{x} + b)\right].
	\end{equation}

	For functions of the form \eqref{intro:cont_nn} with a RELU activation function, the Barron norm is defined as
	\begin{equation}\label{intro:barron_norm}
		\|g \|_{\mathcal{B}} \coloneq \inf_{\mu \in \Pw} \max_{\left( \bm{w},b,c \right) \in \text{supp}\left( \mu \right) } |c| \left( \|\bm{w} \|_{1} + |b|\right).
	\end{equation}
	Functions with a finite Barron norm \eqref{intro:barron_norm} form the Barron space~\cite{barronUniversalApproximationBounds1993,eBarronSpaceFlowInduced2022}. The Barron norm can be bounded, with $ G $ denoting the Fourier transformation of $ g $, as 
	\begin{equation}
		\|g \|_{\mathcal{B}} \leq 2\inf_{G} \int_{\R^{d}} \| \omega \|^{2}_{1} \left| G(\omega) \right| d\omega + 2 \|\nabla g(0) \|_{1} + 2  | g(0) |,
	\end{equation}
	which shows that sufficiently regular functions are in the Barron space. Two-layer neural networks $ \hat{g}_{M} $ can approximate functions $ g $ in the Barron space up to an arbitrary precision, where the approximation error satisfies
	\begin{equation}\label{intro:approx_bound}
		\| g \left( \cdot \right) - \hat{g}_{M}\left( \cdot; \bm{\theta} \right)  \|^{2} \leq \frac{3\|g \|^{2}_{\mathcal{B}}}{M}.
	\end{equation}
	This bound implies an $\mathcal{O}(M^{-1/2})$ approximation rate in the function-space
	norm, established in~\cite{barronUniversalApproximationBounds1993} (Theorem~1) and
	extended in~\cite{eBarronSpaceFlowInduced2022} (Theorem~1).

	With the measure-based viewpoint of neural networks, introduced in Eq.~\eqref{intro:cont_nn}, we can consider continuous optimization schemes and analyse the training of infinitely wide neural networks. Section \ref{section5} elaborates further on a particle-based method for optimizing infinitely wide neural networks.


    \section{Optimization Methods}\label{section3}
	In this section, we formulate the general optimization problem underlying neural network training. Next, we present the Adam optimizer and the Consensus-Based Optimization method. Lastly, we introduce two variants of the CBO method, namely a hybrid method that combines Adam and CBO and a Multi-Task CBO formulation.  

	\subsection{General Training Problem}
	The central goal in machine learning is to approximate an unknown function $ g: X \to Y$ based on observed data. We assume that $ g $ can be well represented by two-layer neural networks with parameters $ \btheta $. Training the neural network then amounts to finding parameters $ \bm{\theta} $ such that neural network $ \hat{g}\left( \bm{x} ; \btheta \right)$ approximates the unknown function as accurately as possible. 
	
	The quality of the approximation is typically measured with a loss function $ \mathcal{L} $, which quantifies the discrepancy between the neural network prediction $ \hat{\bm{y}}_{s} = \hat{g}\left( \bm{x}_{s};\bm{\theta}\right)  $ and the observed label $ \bm{y}_{s} $. From a theoretical perspective, training a neural network can be understood as minimizing the risk function $ R(\btheta) $. The risk function is defined as the expected loss with respect to a data distribution $ P(\bm{x},\bm{y}) $:
	\begin{equation}
		R(\btheta) = \int_{X \times Y} \mathcal{L}(\bm{y}, \hat{g}(\bm{x};\btheta)) dP(\bm{x},\bm{y}).
	\end{equation}
	In practice, the true data distribution is not known. Therefore, the risk function $ R(\btheta) $ is approximated by the empirical risk
	\begin{equation}\label{eq:empirical_risk}
		\hat{R}\left( \bm{\theta} \right) = \frac{1}{S}\sum_{s=1}^{S}\mathcal{L}\left( \bm{y}_{s},\hat{g}(\bm{x}_{s};\btheta) \right),
	\end{equation}
	where $\left\{\left( \bm{x}_{s},\bm{y}_{s} \right)\right\}_{s=1}^{S}$ denotes the training dataset~\cite{vapnik1999}.

	The choice of $ \mathcal{L} $ depends on the particular task. For regression, where $ Y = \R $, the squared error loss
	\begin{equation}\label{eq:MSE}
		\mathcal{L}\left( y,\hat{y}\right) = \left( y - \hat{y} \right)^2,
	\end{equation}
	leads to the Mean Squared Error (MSE) risk~\cite{goodfellowDeepLearning2016}. For classification with $ C $ classes, the standard choice is the cross-entropy loss function
	\begin{equation}\label{eq:cross-entropy}
		\mathcal{L}\left( \bm{y},\hat{\bm{y}} \right) = - \sum_{c=1}^{C}y_{c}\log(\hat{y}_{c}),
	\end{equation}
	where $ y $ is a one-hot encoded label and $ \hat{y} $ a probability vector obtained after applying the softmax function to the output of the neural network~\cite{murphyMLProb2014}. There exist many other loss functions in machine learning, such as hinge loss, Huber loss and logistic loss.
	
	Training a neural network can thus be formulated as an optimization problem, where we aim to find
	\begin{equation}
		\btheta^* = \argmin_{\btheta} \hat{R}\left( \btheta \right), 
	\end{equation}
	with the empirical risk $ \hat{R}\left( \btheta \right)  $ defined as in Eq.~\eqref{eq:empirical_risk}. The empirical risk function is typically nonconvex due to the nonlinear structure of the neural networks~\cite{goodfellowDeepLearning2016}. Standard optimization methods in machine learning rely on the gradient of the empirical risk, with Adam being the most popular~\cite{kingma2015adam}. However, gradient-based methods tend to converge to local minima in nonconvex landscapes. Particle-based optimization methods, such as Consensus-Based Optimization, are known as global optimization methods, for which convergence to the global minimizer of certain nonconvex functions is theoretically guaranteed.
	
	\subsection{Gradient-Based Optimization}\label{gradient_based_opt}
	Nearly all neural networks are trained with  gradient-based methods. The most popular are by far Stochastic Gradient Descent (SGD) and adaptive variants such as Adam. In what follows, we provide a brief overview of these two methods, which form the foundation of modern neural network training.

	Stochastic Gradient Descent is an adaptation of the classical Gradient Descent method~\cite{goodfellowDeepLearning2016}. Gradient Descent is an iterative method that updates the model parameters in the direction of the negative gradient of the objective function~\cite{NocedalNumericalOptimization}. While classical Gradient Descent often employs a line search to determine an optimal step size, it is common in neural network training to use a fixed step size for simplicity. At iteration $ k $, the step direction $ \bm{d}^{k} $ is given by
	\begin{equation}\label{eq:gradient_descent}
		\bm{d}^{k} = \frac{1}{S}\sum_{s=1}^{S} \nabla_{\btheta} \mathcal{L}\left(\bm{y}_{s}, \hat{g}(\bm{x}_{s};\btheta^{k}) \right),
	\end{equation} 
	and the parameters are updated according to
	\begin{equation}\label{eq:optimization_method}
		\btheta^{k+1} =  \btheta^k - \Delta t \bm{d}^{k},
	\end{equation}
	where $ \Delta t > 0 $ denotes the time step or the learning rate. A major drawback of Gradient Descent is that computing the step direction at each iteration scales linearly with the dataset size, $\mathcal{O}(S)$. Datasets in machine learning can easily contain up to one million data points, making Gradient Descent computationally expensive. 
	
	SGD alleviates this issue by considering a minibatch of data points $ \{(\bm{x}_{s}, \bm{y}_{s}) \}_{s=1}^{S^\prime} $ drawn uniformly from the dataset. The minibatch size $ S^\prime \ll S$ is chosen before training. The step direction $ \bm{d}^{k} $ in SGD equals
	\begin{equation}\label{eq:stochastic_gradient}
		\bm{d}^{k} = \frac{1}{S^\prime}\sum_{s=1}^{S^\prime} \nabla_{\btheta} \mathcal{L}\left(\bm{y}_{s}, \hat{g}(\bm{x}_{s};\btheta^{k}) \right),
	\end{equation}
	which is an unbiased estimator of the full gradient \eqref{eq:gradient_descent}. At iteration $ k $, the gradient is computed using the current minibatch, and the parameters are subsequently updated. At the next iteration $ k+1 $, a new minibatch is sampled. One pass through the complete training dataset, where each data point has been used once for updating the parameters, is referred to as an epoch.

	The performance of SGD depends heavily on the choice of the learning rate $ \Delta t $. There is a trade-off: a larger learning rate yields faster progress, but is more unstable; a smaller learning rate improves stability but converges slower. In practice, the learning rate decays during training according to a predefined schedule. However, there also exist adaptive algorithms, such as Adam, that adapt the learning rate individually for each parameter based on the gradient history. Given the stochastic gradient $ \bm{d}^{k} $ from Eq.~\eqref{eq:stochastic_gradient}, Adam estimates the first and second moment as
	\begin{align}\label{eq:adam_moments}
		&\bm{s}^{k+1} = \beta_1 \bm{s}^k + (1 - \beta_1)\bm{d}^{k}
		&\bm{r}^{k+1} = \beta_2 \bm{r}^k + (1 - \beta_2)\bm{d}^{k} \odot \bm{d}^{k},
	\end{align}
	with the decay parameters $ \beta_1, \beta_2 \in [0,1)$ and where $ \odot $ represents the elementwise multiplication. The moment estimates are normalized to correct for the initial bias:
	\begin{align}\label{eq:adam_moments_normal}
		&\hat{\bm{s}}^{k+1} = \frac{\bm{s}^{k+1}}{1 - \beta_{1}^{k+1}}
		&\hat{\bm{r}}^{k+1} = \frac{\bm{r}^{k+1}}{1 - \beta_{2}^{k+1}}.
	\end{align}
	Finally, Adam updates the parameters $ \btheta $ as follows
	\begin{equation}
		\btheta^{k+1} = \btheta^{k} -\Delta t~\frac{\hat{\bm{s}}^{k+1}}{\sqrt{\hat{\bm{r}}^{k+1}}+\delta},
	\end{equation}
	with $ \delta $ a small constant for numerical stability. The initial values for the first and second moment estimates $ \bm{s}^{0} $ and $ \bm{r}^{0} $ are set to zero. By adapting learning rates per parameter, Adam typically achieves faster convergence than plain SGD. 
	
	\subsection{Consensus-Based Optimization}\label{cbo}
	We introduce Consensus-Based Optimization, a global optimization method well suited for nonconvex, nonsmooth objective functions~\cite{pinnauConsensusbasedModelGlobal2017a}. We aim to minimize the empirical risk \eqref{eq:empirical_risk}. To find the minimum, we consider an ensemble of $ N $ particles $ \bm{\theta}^{k}_{n} \in \R^{M(d+2)},~n=1, \dots, N$ with $ N \in \mathbb{N}$ at time step $ k $. The state of the particles evolves according to the discretized stochastic differential equation 
	\begin{equation}\label{cbo:sde}
		\bm{\theta}_{n}^{k+1} = \bm{\theta}_{n}^{k} -\lambda \Delta t \left(\bm{\theta}_{n}^{k} - \bm{V}^{k} \right) + \tilde{\sigma} \sqrt{\Delta t}\left( \bm{\theta}_{n}^{k} - \bm{V}^{k} \right)\odot\bm{\xi}_{n}^{k}, \qquad \bm{\xi}_{n}^{k} \sim \mathcal{N}\left( \bm{0}, \bm{\mathrm{I}} \right),
	\end{equation} 
	with $ \lambda, \tilde{\sigma} > 0$ representing the drift and diffusion parameters respectively. The dynamics combine a drift towards the consensus point with a diffusion term that promotes exploration. The consensus point $ \bm{V} \in \R^{M(d+2)}$ is calculated as a weighted average
	\begin{equation}\label{cbo:consensus_point}
		\bm{V}^{k} = \frac{\frac{1}{N}\sum_{n=1}^{N}\bm{\theta}^{k}_{n}\exp\left( -\alpha \hat{R}(\bm{\theta}^{k}_{n})\right)}{ \frac{1}{N}\sum_{n=1}^{N}\exp\left( -\alpha \hat{R}(\bm{\theta}^{k}_{n})\right)} = \sum_{n=1}^{N} \beta(\bm{\theta}^{k}_{n}) \bm{\theta}^{k}_{n}, \quad \beta(\bm{\theta}^{k}_{n}) = \frac{\exp(-\alpha \hat{R}(\bm{\theta}^{k}_{n}))}{\sum_{n=1}^{N}\exp(-\alpha \hat{R}(\bm{\theta}^{k}_{n}))},
	\end{equation}
	where the weight $ \beta(\bm{\theta}^{k}_{n}) $  depends on the relative performance of the particle in the optimization landscape. In statistical physics, the parameter $ \alpha > 0 $ represents the inverse temperature. The initial positions of the particles $ \bm{\theta}_{n}^{0} $ are distributed independently and identically according to a chosen initial distribution $ \rho^0 $.

	The computation of the consensus point $ \bm{V}^{k} $ requires evaluating the empirical risk for each particle $ \btheta_{n}^{k} $. This process can become computationally demanding when the number of particles or the size of the training set is large. To reduce the cost, a minibatch strategy analogous to that used in SGD can be applied within CBO. Let $ \{(\bm{x}_{s}, \bm{y}_{s}) \}_{s=1}^{S^\prime} $, $~ S^\prime \ll S  $, denote a minibatch of training data drawn uniformly from the training dataset. The corresponding minibatch empirical risk is then defined as
	\begin{equation}
		\hat{R}(\btheta_{n}^{k}) = \frac{1}{S^\prime} \sum_{s=1}^{S^\prime}\mathcal{L}\left( \bm{y}_{s},\hat{g}(\bm{x}_{s};\btheta_{n}^{k}) \right).
	\end{equation}
	At iteration $ k $, the consensus point is computed using the current minibatch, and the particle positions are subsequently updated. At the next iteration $ k+1 $, a new minibatch is sampled. One pass through the complete training dataset, where each data point has been used once for updating the particles, is referred to as an epoch.

	The consensus-based optimization scheme converges under certain conditions. In~\cite{carrillo2021}, the authors show that for high dimensional problems and anisotropic noise, we need $ 2\lambda > \tilde{\sigma}^2 $ to form consensus, which is independent of the optimization dimension. Secondly, we note that the consensus point in \eqref{cbo:consensus_point} is a finite particle discretization of the mean of a Gibbs-type measure. By the Laplace principle~\cite{dembo2010} 
	\begin{equation}
		\lim_{\alpha \to \infty} \left( - \frac{1}{\alpha} \log \left( \int \exp \left( -\alpha \hat{R}(\bm{\theta}) \right)\rho \left(\bm{\theta}\right) d\bm{\theta} \right) \right) = \min_{\theta} \hat{R}(\bm{\theta}),
	\end{equation}
	the Gibbs measure concentrates exponentially around the global minimizer of the objective function. Consequently, the consensus point will converge towards the global minimizer, provided that $ \bm{\theta}^{*} \in \text{supp}\left(\rho^0\right)$~\cite{carrillo2018}. The particle-based nature of CBO enables the formulation of a mean-field limit for $ N \to  \infty$. For analytical results on the mean-field limit, we refer to~\cite{carrillo2018,gerber2025,huang2022,kossMeanFieldLimit2024a, fornasier2022, gerber2026uniform}.

	\subsection{Variations of the CBO Method}
	We present two new variations of the classical Consensus-Based Optimization method, designed to improve performance of the classical CBO method. We first present the hybrid method, which improves the convergence speed of CBO and secondly Multi-Task CBO, which reduces memory overhead.

	\subsubsection{Hybrid Method}
	A natural approach to develop a new method is to combine two existing methods. We devise a hybrid method where we combine an Adam step and a CBO step together, similar to~\cite{wei2025}. Formally, the update is given by
	\begin{equation}
	\scalemath{0.95}{
    \bm{\theta}^{k+1}_{n} = \bm{\theta}^{k}_{n} - \gamma\Delta t~\frac{\hat{\bm{s}_{n}}^{k+1}}{\sqrt{\hat{\bm{r}}_{n}^{k+1}} + \delta} + \left( 1 - \gamma \right)\left( - \lambda \Delta t\left(\bm{\theta}_n^k - \bm{V}^k \right) + \tilde{\sigma} \sqrt{\Delta t}\left(\bm{\theta}_n^k - \bm{V}^k \right) \odot \bm{\xi}_n^k \right)},
	\end{equation}
	where $ \gamma $ is a parameter between 0 and 1 and controls how much influence CBO or Adam have on the update. The first and second moment estimates are computed as in \eqref{eq:adam_moments} and \eqref{eq:adam_moments_normal}, with the gradient for each particle given by
	\begin{equation}
		\bm{d}_{n}^{k} = \frac{1}{S}\sum_{s=1}^{S} \nabla_{\btheta} \mathcal{L}\left(\bm{y}_{s}, \hat{g}(\bm{x}_{s};\btheta_{n}^{k}) \right).
	\end{equation} 
	In the hybrid method, the minibatch strategy involves using the same minibatch to compute the gradient and the consensus point, as explained in Subsections \ref{gradient_based_opt} and \ref{cbo}.

	\subsubsection{Multi-Task CBO}

	Multi-task learning is another area of machine learning where the CBO method can be naturally applied to. Multi-task learning is a multi-objective optimization problem with the goal to improve the generalization of a model~\cite{caruana1997}. Instead of building a separate model for each task, one model shares parts of its internal representation. The knowledge gained from one task can improve the performance of other tasks. In practice, the multi-task risk is often of the form 
	\begin{equation}
		\hat{R} = \sum_{p} \kappa_{p}\hat{R}\left( \mathcal{T}_{p} \right),
	\end{equation}
	with $ \hat{R}\left( \mathcal{T}_{p} \right) $ the empirical risk associated to task or problem $ \mathcal{T}_{p} $ and a weight $ \kappa_{p} $. Currently, multi-task models are trained with a gradient method, but this method has several drawbacks. One key aspect is that the gradients have to be balanced by updating the weights $ \kappa_{p}$ to avoid one specific task from dominating the training of the model~\cite{chen2018gradnorm}. Another issue is gradients from different tasks conflicting with one another~\cite{yu2020gradient}. This can be alleviated by projecting one of the gradients of a task onto another gradient.

	The CBO method allows for a natural and memory-efficient implementation of the multi-task training~\cite{borghi2022}. Figure \ref{fig:diagram_multi_task} provides a conceptual illustration. Consider the training of a single neural network with CBO to approximate a given function. A requirement for convergence to the global minimizer is that $ \bm{\theta}^{*} \in \text{supp}\left( \rho^0 \right) $. Accordingly, we initialize particles by sampling from the distribution $ \rho^0 $. 
	
	Now consider a second, related approximation problem to the first. If the task is sufficiently similar, we expect the corresponding global minimizer to lie close to the first minimizer, in particular, within the support of $ \rho^0 $, as illustrated in Figure \ref{fig:diagram_multi_task}. In this case, there is no need to sample additional new particles for the second task; the particles used for the first task can be reused for the second task. More generally, the particles can be recycled across multiple tasks, provided that the global minimizers of these tasks remain within $ \text{supp}\left( \rho^0 \right) $. 
	
	A practical strategy for choosing the number of particles is to set $N$ equal to the number of tasks. In that case, an increasing number of tasks results in improved optimization accuracy, since each empirical risk function $ \hat{R}\left( \mathcal{T}_p \right)  $ is optimized with a larger population of particles. In the multi-task context, CBO does not need a particle ensemble for each task, reducing the memory overhead. However, a larger particle count than the number of tasks is also feasible in practice.

	\begin{figure}[h]
		\includegraphics[width = \linewidth]{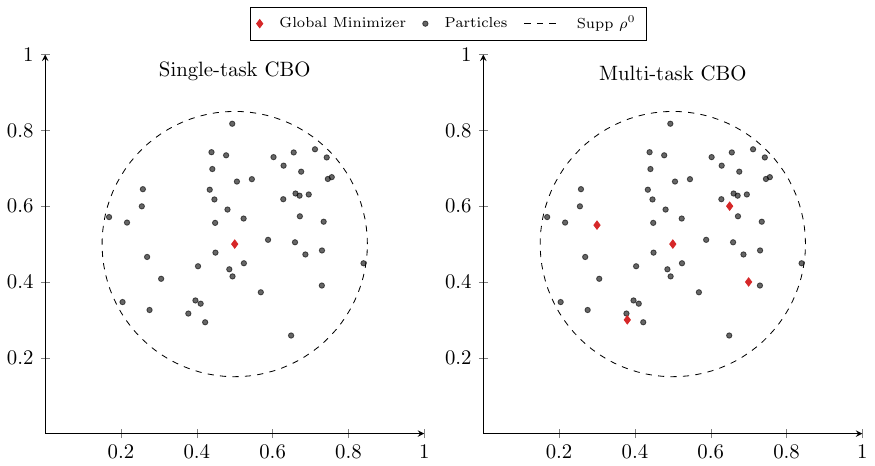}
		\captionsetup{justification=centering}
		\caption{Conceptual illustration of single-task versus Multi-Task CBO. In single-task CBO a single consensus point guides all particles toward the global minimizer of one empirical risk. In Multi-Task CBO the same particle ensemble is recycled across related tasks whose minimizers lie within the support of the common initialization $ \rho^0 $.}
		\label{fig:diagram_multi_task}
	\end{figure}

	We can reformulate the CBO method in the multi-task setting as 
	\begin{equation}\label{eq:multi_task_cbo}
	\begin{cases}
	\bm{\theta}_{n}^{k+1} = \bm{\theta}_{n}^{k} -\lambda \Delta t \left(\bm{\theta}_{n}^{k} - \bm{V}\left(\mathcal{T}_{p}^{k}\right) \right) + \tilde{\sigma} \sqrt{\Delta t}\left( \bm{\theta}_{n}^{k} - \bm{V}\left(\mathcal{T}_{p}^{k}\right) \right)\odot\bm{\xi}_{n}^{k}\\
	\mathcal{T}_{p}^{k+1} = \mathcal{T}_{p}^{k}\\
	\end{cases}, \quad n = p = 1, \dots, N,
	\end{equation}
	where the number of particles $ N $ equals the number of tasks. The consensus point is given by
	\begin{equation}
		\bm{V}\left(\mathcal{T}_{p}^{k}\right) = \frac{\frac{1}{N}\sum_{n=1}^{N}\bm{\theta}^{k}_{n}\exp\left( -\alpha \hat{R}\left(\bm{\theta}^{k}_{n}, \mathcal{T}_{p}^{k}\right)\right)}{ \frac{1}{N}\sum_{n=1}^{N}\exp\left( -\alpha \hat{R}\left(\bm{\theta}^{k}_{n}, \mathcal{T}_{p}^{k}\right)\right)}.
	\end{equation}


	\section{Numerical Examples}\label{section4}

	In this section, we present three numerical experiments designed to illustrate and evaluate the \mbox{optimization} methods discussed in Section \ref{section3}. The implementation is publicly available online~\cite{WilliamCode}.

	\subsection{Example 1: Sine Approximation}
	The goal of the first experiment is to investigate the applicability of CBO to train two-layer neural networks. To this end, we compare CBO with Adam, focusing on both the minimal empirical risk obtained and the stability of the training procedure. 
	
	The experiment setup is the following: consider a one-dimensional regression problem, where the goal is to approximate the function $x \mapsto \sin(2\pi x) $ on the domain $ [0,1] $ using a finite two-layer neural network, as in Eq.~\eqref{intro:disc_nn}. The neural network has a RELU activation function. We take the width of the neural network $ M = 100$. We sample 8000 data points $ x_{s} $ uniformly on the interval $[0,1]$ and apply the function $ y_{s} = \sin(2\pi x_{s}) + 0.01 \xi_s$, with $ \xi_s \sim \mathcal{N}(0,1) $, to generate the training dataset $\left\{\left( x_{s}, y_{s} \right)\right\}_{s=1}^{8000}$. We apply the minibatch strategy to both Adam and CBO using a minibatch size of $ S^\prime $. We choose the squared error loss, as in~\eqref{eq:MSE}, for the loss function. We initialize the particles of the CBO method from the uniform distribution $ \bm{\theta}^{0}_{n} \sim \mathcal{U}[-1,1] $. For Adam, we keep the default initialization of the neural network parameters provided by PyTorch~\cite{Pytorch}.
	
	The training of the two-layer neural network with CBO is carried out with the following parameters:
	\begin{align*}
		\quad N = 200, \quad \Delta t = 0.1, \quad \alpha = 10^5, \quad \lambda = 1, \quad \tilde{\sigma} = \sqrt{1.6}, \quad S^\prime = 800.
	\end{align*}
	Parameters shared between CBO and Adam, such as the minibatch size $ S^\prime $ and the learning rate (or equivalently time step) $ \Delta t $, are always taken equal. Every 100 epochs, the parameter $ \alpha $ is multiplied by 10 until it reaches $ 10^7 $. Lastly, we run the experiment ten times with different seeds and present the median of the results.

	Figure \ref{fig:sine_training} shows the empirical risk per epoch, comparing the optimization performance of CBO and Adam. In Figure \ref{fig:sine_example}, we present the approximation obtained with a two-layer neural network trained with CBO and Adam.
	
	In Figure \ref{fig:sine_training}, we observe that the CBO method converges to a lower empirical risk than Adam. The CBO method also exhibits better stability than Adam. Figure \ref{fig:sine_example} confirms that both CBO and Adam find a good approximation of the sine function. We do note, however, that the computational cost per iteration is higher for CBO, since a forward pass through the neural network is required for each particle.

	\begin{minipage}[t]{0.48\linewidth}
    \centering
	\includegraphics[width = \linewidth]{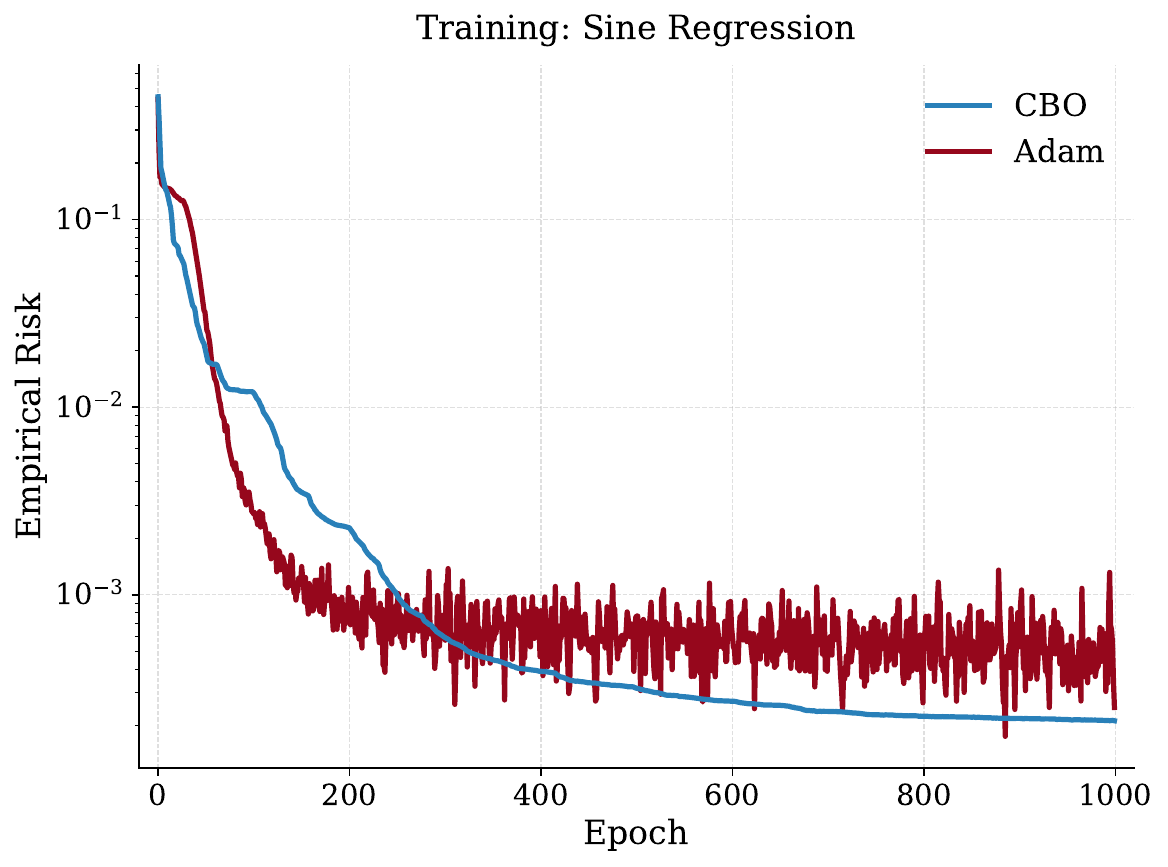}
	\captionsetup{justification=centering}
	\captionof{figure}{Empirical risk $ \hat{R}(\btheta) $ as a function of training epochs for a two-layer neural network trained with Adam and CBO. The figure displays the median empirical risk taken over 10 simulations.}
	\label{fig:sine_training}
	\end{minipage}
	\hfill
	\begin{minipage}[t]{0.48\linewidth}
    \centering
    \includegraphics[width = \linewidth]{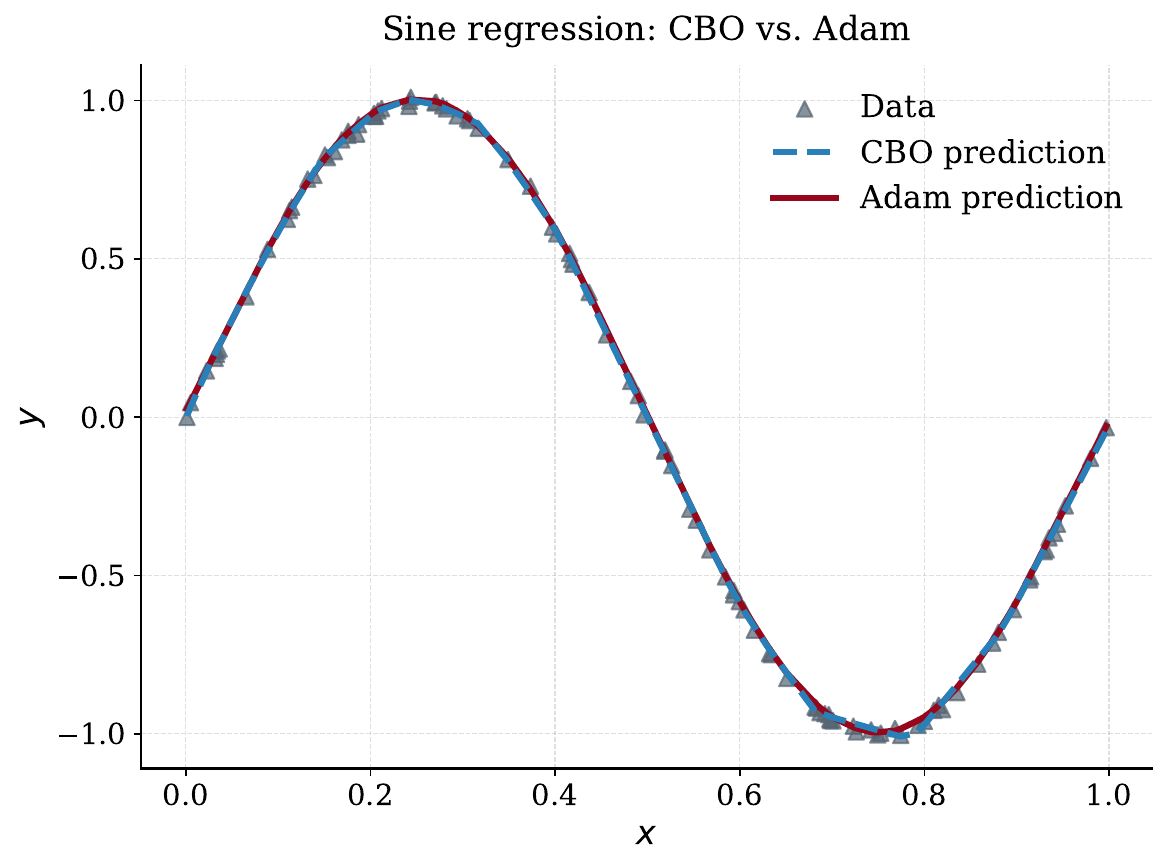}
	\captionsetup{justification=centering}
	\captionof{figure}{Approximations of the sine function $ \sin(2 \pi x) $ obtained with a two-layer neural network with $ M = 100 $ trained with CBO and Adam. The plot only contains a subset of the training dataset to improve the clarity.}
	\label{fig:sine_example}
	\end{minipage}
	
	\subsection{Example 2: MNIST}
	In the second experiment, we further investigate the applicability of CBO in training two-layer neural networks to classify the MNIST dataset~\cite{lecun1998mnistdatabase}. This is a standard classification problem in machine learning, and commonly used to benchmark new methods and models. We compare CBO, Adam and the hybrid method on both the minimal empirical risk achieved and the stability of the training. 
	
	The experiment setup is the following: consider the multi-class classification of the MNIST dataset with $ C = 10 $ classes. We again use a two-layer neural network with a width $ M = 20 $ and the RELU activation function. The MNIST dataset contains grayscale images of $ 28 \times 28 $ pixels. We flatten these images to vectors with 784 elements, which represent the input data points $ \bm{x}_s$. In the MNIST dataset, the true labels are given as class indices, i.e., $y_s \in \{1, \dots, C \} $. We take a subset of $ 10~000 $ images to form the training dataset $\left\{\left( \bm{x}_{s}, y_{s} \right)\right\}_{s=1}^{10000}$. We apply the minibatch strategy to Adam, CBO and the hybrid method using a minibatch size of $ S^\prime $. The neural network outputs a ten dimensional vector $ \hat{\bm{y}}_s $. We choose the cross-entropy loss function of PyTorch, similar to Eq.~\eqref{eq:cross-entropy}, which internally applies the softmax function to the predictions~\cite{Pytorch}. We initialize the particles of the CBO method from the uniform distribution $ \bm{\theta}^{0}_{n} \sim \mathcal{U}[-1,1] $. For Adam, we keep the default initialization of the neural network parameters provided by PyTorch.
	
	The training of the two-layer neural network with CBO is carried out with the following parameters:
	\begin{align*}
		N = 1000, \quad \Delta t = 10^{-1}, \quad \alpha = 10^5, \quad \lambda = 1, \quad \tilde{\sigma} = \sqrt{1.4}, \quad S^\prime = 1000.
	\end{align*} 
	For the Adam method, we set the minibatch size equal to that of CBO and perform a
	parameter sweep over the learning rates $\{10^{-1},\, 10^{-2},\, 10^{-3},\, 10^{-4}\}$
	to determine the optimal value. Adam is stable from $\Delta t = 10^{-3}$ onwards;
	we therefore adopt this value.

	The parameters of the hybrid method are set as:
	\begin{align*}
		N = 1000, \quad \Delta t = 10^{-2}, \quad \alpha = 10^5, \quad \lambda = 1,\quad \tilde{\sigma} = \sqrt{1.4}, \quad S^\prime = 1000, \quad \gamma = 0.7.
	\end{align*}

	Figure \ref{fig:comparison_cbo_adam} illustrates the empirical risk per epoch for CBO, Adam and the hybrid method. We observe that Adam obtains a significantly lower empirical risk than CBO and converges faster. The hybrid method converges faster than either method individually. The larger learning rate ($\Delta t = 10^{-2}$) used in the hybrid method, at which Adam is unstable, is admissible because the CBO component improves the stability.

	\begin{figure}[ht]
    \centering
    \includegraphics[width = 0.8\textwidth]{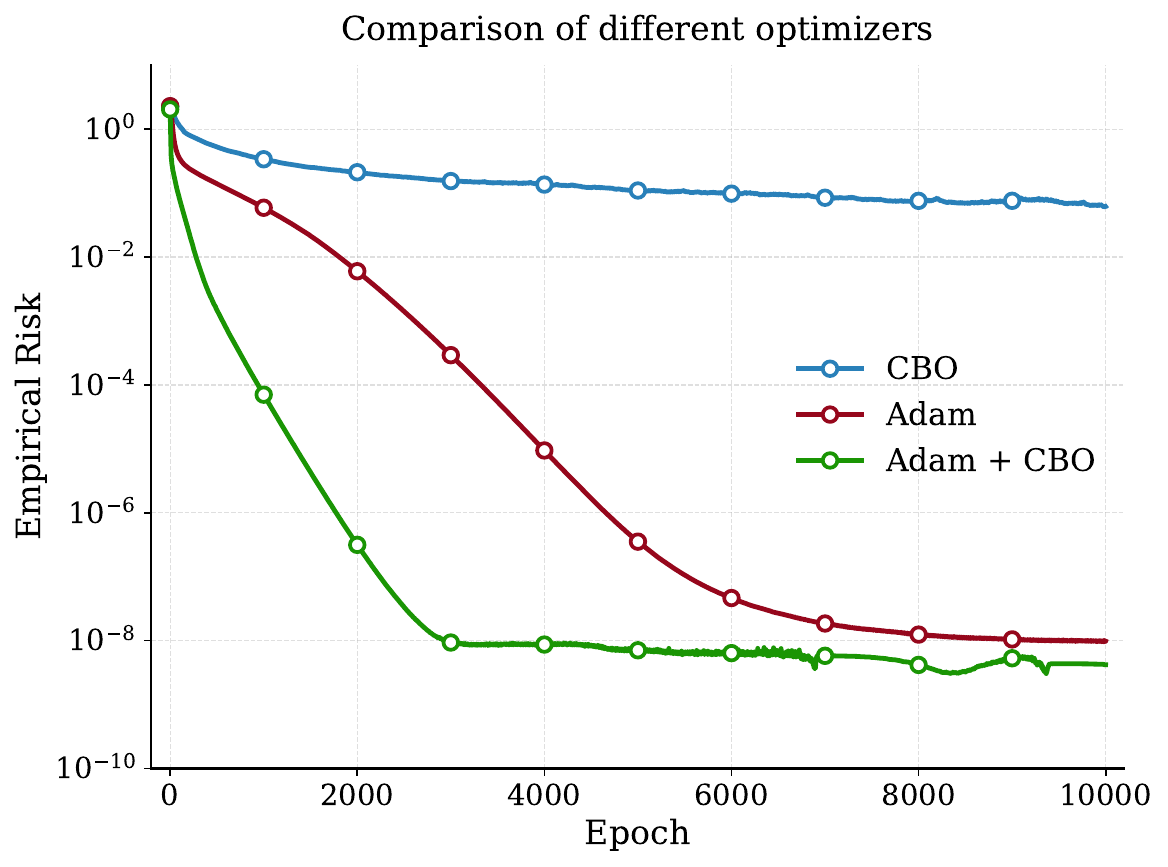}
	\captionsetup{justification=centering}
	\captionof{figure}{Empirical risk $ \hat{R}(\btheta) $ as a function of training epochs for a two-layer neural network trained with Adam ($\Delta t = 10^{-3}$), CBO ($\Delta t = 10^{-1}$) and the hybrid method (Adam + CBO) ($\Delta t = 10^{-2}$) on the MNIST dataset.}
	\label{fig:comparison_cbo_adam}
    \end{figure}

	The improved stability of the hybrid method is a consequence of the fact that the minimizer is represented by the consensus point $ \bm{V}^k $. For the consensus point to suddenly jump due to an instability, all particles would need to move simultaneously in approximately the same direction. However, before convergence, the particles are distributed across the optimization landscape, so they do not experience the same instability simultaneously; interaction through the consensus point stabilizes the dynamics. 

	\subsection{Example 3: Multi-Task CBO}
	In the third example, we study whether the particle recycling strategy of Multi-Task CBO is able to \mbox{minimize} multiple empirical risk functions. Specifically, we assess whether the empirical risk consistently decreases throughout training for all considered tasks.
	
	We set up the experiment as follows: consider the approximation problem of 100 shifted sine functions, where the goal is to approximate each function using a different two-layer neural network. We take 100 functions of the form $(x,\Delta y_{p}) \mapsto \sin(2\pi x) + \Delta y_{p} $ on the domain $[0,1]$, where the shifts are uniformly spaced in $[-1,1]$, $ \Delta y_{p} = -1 + 2(p-1)/99,~p=1, \dots, 100 $. We assume that the shifted sine functions admit similar neural network representations and hence we expect their corresponding global minimizer to be in close proximity to each other.
	
	The neural networks have a RELU activation function and a width of $ M = 100 $. We sample 8000 data points $ x_{s} $ uniformly on the interval $[0,1]$ and apply the function $ y_{s, p} = \sin(2\pi x_{s}) + \Delta y_{p}$ to generate 100 training datasets $\left\{\left( x_{s}, y_{s, p} \right)\right\}_{s=1}^{8000},~p=1, \dots, 100$. We apply the minibatch strategy, dividing each dataset into minibatches of size $ S^\prime $. The loss function for each task $p$ is the squared error loss, resulting in 100 different MSE risks $ \hat{R}_{p} $. We initialize the particles of the CBO method from the uniform distribution $ \bm{\theta}^{0}_{n} \sim \mathcal{U}[-1,1] $. 
	
	The parameters of the experiment are:
	\begin{align*}
		N = 200, \quad \Delta t = 0.2, \quad \alpha = 10^4, \quad \lambda = 1, \quad \tilde{\sigma} = \sqrt{1.8}, \quad S^\prime = 800.
	\end{align*}
	Every 100 epochs, the parameter $ \alpha $ is multiplied by 10 until it reaches $ 10^7 $. Lastly, we run the experiment ten times to average out as much noise.

	Given 100 different problems, the Multi-Task CBO method has 100 different consensus points. When optimizing with 200 particles, it is necessary to determine which particles move towards which consensus point. The particle update strategy is as follows: the first two particle will move towards the first consensus point, the third and fourth particles will move to the second consensus point and so forth. Multi-Task CBO effectively assigns two particles to each problem.
	  
	Figure \ref{fig:multi_task_training} presents the median and minimum empirical risk per epoch. The median and minimum are taken over the 100 tasks during training. In Figure \ref{fig:multi_task_examples}, we display the approximation results for five problems. In Figure \ref{fig:multi_task_training}, we observe that both the minimum and median decrease during training, indicating that the CBO method effectively minimizes the risk for all tasks. In Figure \ref{fig:multi_task_examples}, we observe an accurate approximation for these five problems, confirming that CBO successfully trained multiple tasks using the same particle set.

	\begin{minipage}[t]{0.48\linewidth}
    \centering
    \includegraphics[width = \linewidth]{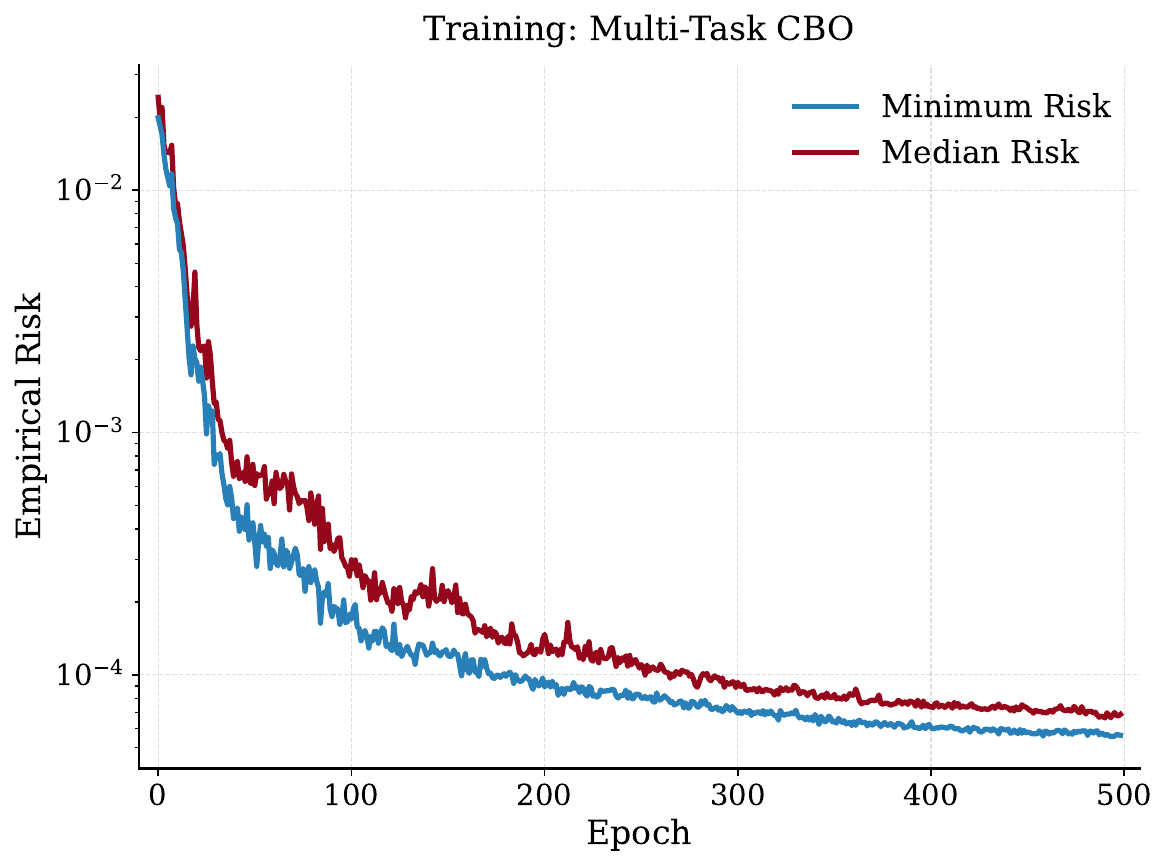}
	\captionsetup{justification=centering}
	\captionof{figure}{The median and minimum empirical risk $ \hat{R}(\btheta) $ as a function of training epochs for two-layer neural networks trained with Multi-Task CBO. The median and minimum are taken over 100 different risk functions.}
	\label{fig:multi_task_training}
	\end{minipage}
	\hfill
	\begin{minipage}[t]{0.48\linewidth}
    \centering
    \includegraphics[width = \linewidth]{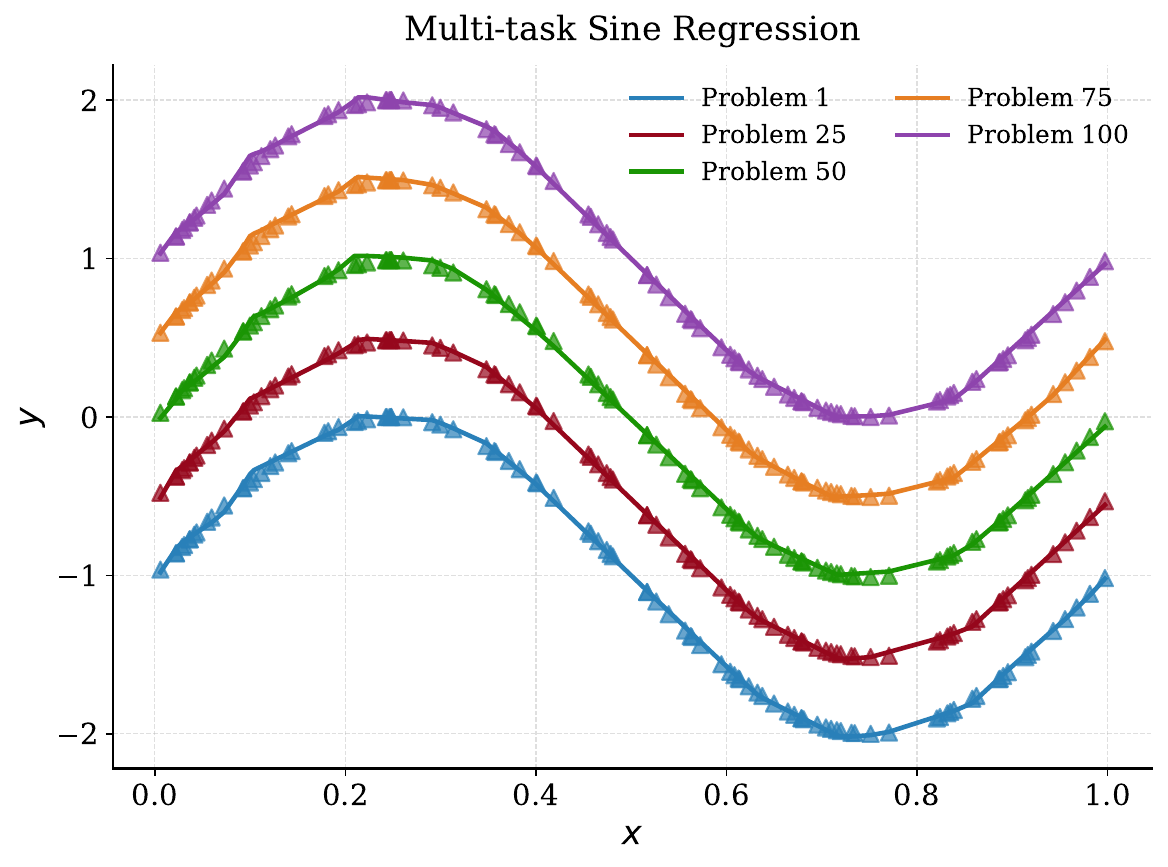}
	\captionsetup{justification=centering}
	\captionof{figure}{Approximation of five different sine functions obtained with five different two-layer neural networks, each with $ M = 100 $. The neural networks are trained with Multi-Task CBO.}
	\label{fig:multi_task_examples}
	\end{minipage}

    \section{Mean-field Models}\label{section5}
	This section provides a study of mean-field models, as indicated in Figure \ref{fig:meanfield_diagram}. We derive an explicit JKO scheme that arises when both the width of the neural network $ M $ and the number of particles $ N $ tend to infinity. In Subsection \ref{limit_M}, we start from the classical CBO formulation, detailed in \mbox{Subsection \ref{cbo}}, and take the width of the neural network to infinity ($ M \to \infty $). We obtain an optimal transport (OT) formulation of CBO. Starting from the optimal transport formulation of CBO, Subsection \ref{limit_N} derives the time-discrete mean-field model ($ N \to \infty $). Further, we show that the population variance decreases each iteration. Finally, Subsection~\ref{sec:mean-field_numerics} illustrates numerically that the empirical risk
	decreases monotonically in $M$ and in $N$ separately, in agreement with
	the respective mean-field models.
	\begin{figure}[h]
		\centering
		\includegraphics[width=0.9\linewidth]{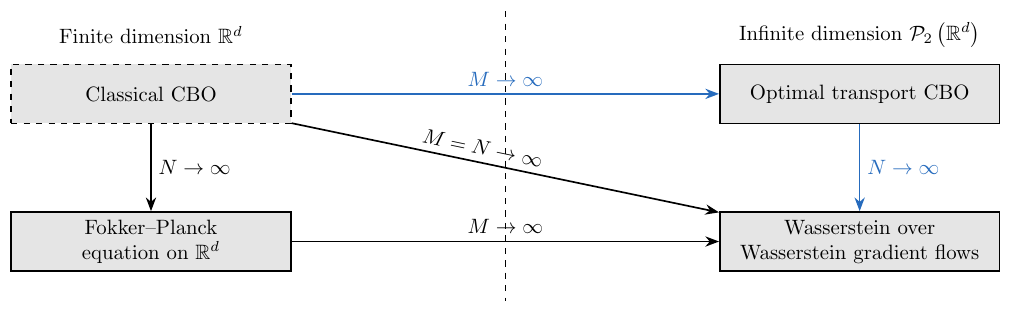}
		\captionsetup{justification=centering}
		\captionof{figure}{Diagram illustrating possible derivations of the mean-field limits. In this paper, we consider the blue path, starting from the classical CBO formulation (dashed box) of Section \ref{section3} and proceeding with $ M $ and then $ N $ to infinity.   }
	\label{fig:meanfield_diagram}
	\end{figure}

	\subsection{The infinite width model ($ M \to \infty $)}\label{limit_M}
	The classical CBO formulation described in Section \ref{section3} formally breaks down when the number of hidden neurons $ M $ tends to infinity, due to the fact that the optimization dimension equals $ M(d+2) $. Instead of representing each neural network as a point in $ \R^{M(d+2)} $, we therefore choose to represent the neural networks by corresponding measures $ \mu_{n} $, as introduced in Section~\ref{section2}. We reformulate the dynamics in the Wasserstein space, similar to the work~\cite{borghi2024}. For an ensemble of $ N $ particles $ \mu^{k}_{n} \in \Pw $ and a time step $ \Delta t \in (0,1] $  we have  
	\begin{equation}\label{sec4:ot_dynamics}
		\mu_{n}^{k+1} = \left( \left( 1 - \Delta t \right)\text{Id} + \Delta t T_{n} \right)_{\#}\mu_{n}^{k},\quad n=1,\dots,N,
	\end{equation}
	with $ T_{n} : \R^{d+2} \to \R^{d+2}$ the optimal transport map defined as 
	\begin{equation}
		\overline{\mu}^{k} = \left( T_{n} \right)_{\#}\mu_{n}^{k}, \quad n=1,\dots,N,
	\end{equation} 
	and $ \overline{\mu}^{k} $ represents the consensus point, i.e., the barycenter in $ \Pw $. It is given by
	\begin{equation}\label{eq:barycenter}
		\overline{\mu}^{k} = \arg \min_{\nu}~ \frac{1}{2}\sum_{n=1}^{N} \beta\left( \mu_{n}^{k}\right)\mathrm{W}_{2}^{2}\left( \nu, \mu_{n}^{k} \right), \quad \beta(\mu_{n}^{k}) = \frac{\exp(-\alpha \hat{R}(\mu_{n}^{k}))}{\sum_{n=1}^{N}\exp(-\alpha \hat{R}(\mu_{n}^{k}))}.    
	\end{equation} 
	We note that the dynamics in \eqref{sec4:ot_dynamics} are fully deterministic. To the best of our knowledge, no formulation of Brownian motion on $ \Pw$ is currently available. As a result, the optimal transport formulation of CBO does not include a diffusion term. However, the noise can be added after discretization with empirical measures $ \hat{\mu}^{k}_{n} $, yielding dynamics that closely resemble the classical CBO formulation. 
	
	We consider the optimal transport formulation of CBO in the single-task setting, although it can also be written in the multi-task setting as
	\begin{equation}
		\begin{cases}
			\mu_{n}^{k+1} = \left( \left( 1 - \Delta t \right)\text{Id} + \Delta t T_{n} \right)_{\#}\mu_{n}^{k}\\
			\mathcal{T}_{p}^{k+1} = \mathcal{T}_{p}^{k}
		\end{cases},\quad n=p=1,\dots,N,
	\end{equation}
	where now the barycenter depends on the particular task $\mathcal{T}_{p}$
	\begin{equation}
		\overline{\mu}\left(\mathcal{T}_{p}^{k}\right) = \arg \min_{\nu}~\frac{1}{2}\sum_{n=1}^{N} \beta\left( \mu_{n}^{k}, \mathcal{T}_{p}^{k}\right)\mathrm{W}_{2}^{2}\left( \nu, \mu_{n}^{k} \right), \quad \beta\left(\mu_{n}^{k}, \mathcal{T}_{p}^{k}\right) = \frac{\exp(-\alpha \hat{R}\left(\mu_{n}^{k}, \mathcal{T}_{p}^{k}\right))}{\sum_{n=1}^{N}\exp(-\alpha \hat{R}\left(\mu_{n}^{k}, \mathcal{T}_{p}^{k}\right))}.
	\end{equation}

	At any given time, it holds that the consensus point and barycenter are applications of a weighted Fréchet mean in the corresponding metric space:
	\begin{proposition}\label{Proposition1}
		The representation of the consensus point in $\R^{d}~\left(Eq.~\eqref{cbo:consensus_point}\right)$ and the barycenter in $ \mathcal{P}_{2}\left( \R^{d} \right)~\left(Eq.~\eqref{eq:barycenter}\right)$ are equal in the following sense:
		\begin{equation}\label{eq:frechet_mean}
			\overline{x} = \arg \min_{v} \frac{1}{2}\sum_{n=1}^{N}\beta \left( x_n \right) d^{2}\left( v, x_{n} \right),
		\end{equation}
		where $ d \left( \cdot, \cdot \right)  $ denotes the distance function in $ \R^{d} $ and $ \mathcal{P}_{2}(\R^{d}) $, respectively.
	\end{proposition}

	\begin{proof}
	The consensus point in $ \R^{d} $ is the result of the minimization of 
		\begin{equation}
			F \left( \bm{v} \right) = \frac{1}{2}\sum_{n=1}^{N}\beta \left( \bm{\theta}_{n} \right) \|\bm{v} - \bm{\theta}_{n} \|_{2}^{2}. 
		\end{equation}
		Given $ F \in C^{1}\left( \R^{d} \right)  $, an optimal point must satisfy the first-order necessary condition
		\begin{equation}\label{eq:fo_necessary}
			0 = \nabla F\left( \bm{v} \right) = - \sum_{n=1}^{N}\beta \left( \bm{\theta}_{n} \right)\left( \bm{\theta}_{n} - \bm{v} \right).
		\end{equation}
		Solving \eqref{eq:fo_necessary} yields the minimizer
		\begin{equation}\label{eq:consensus_linear_form}
			\bm{v}= \sum_{n=1}^{N} \beta \left( \bm{\theta}_{n} \right) \bm{\theta}_{n},
		\end{equation}
		which is the classical consensus point, see Equation \eqref{cbo:consensus_point}. For the barycenter in $ \mathcal{P}_{2}(\R^{d}) $, Equation~\eqref{eq:barycenter} fulfills Proposition~\ref{Proposition1}.
	\end{proof}
  
	\begin{proposition}
		Assume empirical measures of the form
		\begin{equation}
			\hat{\mu}_{n} =  \frac{1}{M}\sum_{i=1}^{M}\delta \left( x - \bm{x}_{n,i} \right) \in \mathcal{P}_{2}(\R^{d}),\quad n=1,\dots,N.
		\end{equation} 
		The barycenter in Eq.~\eqref{eq:barycenter} is given by
		\begin{equation}
			\bar{\mu} =  \frac{1}{M}\sum_{j=1}^{M}\delta \left( y - \bm{y}_{j} \right) \in \mathcal{P}_{2}(\R^{d}),
		\end{equation}
		with support points
		\begin{equation}
		\bm{y}_{j} = \sum_{n=1}^{N}\beta(\hmu)\sum_{i=1}^{M}\left(\pi^{*}_{n}\right)_{j,i}\bm{x}_{n,i},\quad j=1, \dots, M, 
		\end{equation}
		where $ \pi^{*}_{n} \in \R^{M \times M}$ is a permutation matrix.  
	\end{proposition} 

	\begin{proof}
	For any $ M $, let $ \hmu \in \mathcal{P}_{2}\left( \R^{d} \right) $ be
	\begin{equation}\label{eq:empirical_measures}
		\hmu^{k} = \frac{1}{M}\sum_{i=1}^{M}\delta \left( x - \bm{x}_{n,i} \right).
	\end{equation}
	Then, the barycenter is of the form
	\begin{equation}\label{eq:empirical_barycenter}
		\hat{\nu} = \frac{1}{M}\sum_{j=1}^{M}\delta \left( y - \bm{y}_{j} \right),
	\end{equation}
	and according to Proposition \ref{Proposition1}, obtained as the minimization of
	\begin{equation}\label{eq:min_func_measures}
		F \left( \hat{\nu} \right) =  \frac{1}{2}\sum_{n=1}^{N}\beta \left( \hmu \right)W_{2}^{2} \left( \hat{\nu}, \hmu \right).
	\end{equation} 
	Substituting Equations~\eqref{eq:empirical_measures} and \eqref{eq:empirical_barycenter} into \eqref{eq:min_func_measures} yields
	\begin{equation}
		F \left(\bm{y}_{1}, \dots, \bm{y}_{M} \right) = \frac{1}{2}\sum_{n=1}^{N}\beta \left( \hmu \right) \min_{\pi_{n} \in \Pi} \sum_{j=1}^{M} \sum_{i=1}^{M} \frac{1}{M}\left(\pi_{n}\right)_{j,i}\|\bm{y}_{j} - \bm{x}_{n,i} \|_{2}^{2},  
	\end{equation}
	where $ \Pi$ denotes the set of permutation matrices of size $ M \times M $. This set is compact, hence, there exists a minimizer $ \pi^* \in \R^{M \times M}$.  To compute the barycenter, we solve for the optimal coupling $ \pi_{n} $  and the location of the barycenter points $\left(\bm{y}_{1}, \dots, \bm{y}_{M} \right)$~\cite{cuturi2014}. Now, for a fixed optimal coupling $ \pi^{*}_{n} $, we obtain
	\begin{equation}
	F \left(\bm{y}_{1}, \dots, \bm{y}_{M} \right) = \frac{1}{2}\sum_{n=1}^{N}\beta \left( \hmu \right) \sum_{j=1}^{M} \sum_{i=1}^{M} \frac{1}{M}\left(\pi^{*}_{n}\right)_{j,i}\|\bm{y}_{j} - \bm{x}_{n,i} \|_{2}^{2}.
	\end{equation}
	The first-order necessary optimality condition reads
	\begin{equation}
		0 = \frac{\partial F \left(\bm{y}_{1}, \dots, \bm{y}_{M} \right)}{\partial \bm{y}_{j}} = \sum_{n=1}^{N}\beta(\hmu)\sum_{i=1}^{M} \frac{1}{M}\left(\pi^{*}_{n}\right)_{j,i}\left( \bm{y}_{j} - \bm{x}_{n,i} \right),
	\end{equation}
	and hence
	\begin{align}
		\bm{y}_{j} & = \frac{\sum_{n=1}^{N}\beta(\hmu)\sum_{i=1}^{M}\left(\pi^{*}_{n}\right)_{j,i}\bm{x}_{n,i} }{\sum_{n=1}^{N}\beta(\hmu)\sum_{i=1}^{M}\left(\pi^{*}_{n}\right)_{j,i}} = \sum_{n=1}^{N}\beta(\hmu)\sum_{i=1}^{M}\left(\pi^{*}_{n}\right)_{j,i}\bm{x}_{n,i}.
	\end{align}
	\end{proof}

	\subsection{The infinite particle model ($ N \to \infty $)}\label{limit_N}
	In Subsection \ref{limit_M}, we introduced the optimal transport dynamics
	\begin{equation}\label{eq:mean_field_particle_update}
		\mu_{n}^{k+1} = \left( \left( 1 - \Delta t \right)\text{Id} + \Delta t T_{n} \right)_{\#}\mu_{n}^{k},\quad n=1,\dots,N.
	\end{equation}
	with the optimal transport plan $ T_{n} : \R^{d+2} \to \R^{d+2}$ given by
	\begin{equation}
	\bar{\mu}^k = \left( T_{n} \right)_{\#} \mu_{n}^{k}, \quad n=1,\dots,N.
	\end{equation}
	We now consider the mean-field limit as $ N \to \infty$. Let $ X \coloneq \Pw $ and equip $ X $ with the 2-Wasserstein metric $ W_2 $. Define 
	\begin{equation}
		\PPw \coloneq \left\{ \rho \in \mathcal{P}(X): \int_{X} W_2^{2}\left( \mu, \delta_{0} \right) d\rho(\mu) < \infty\right\}.
	\end{equation} 
	For $ \rho, \phi \in \PPw $ define the Wasserstein-over-Wasserstein distance as~\cite{bonet2025flowing,beiglbock2025,pinzi2025}
	\begin{equation}
    \mathbb{W}_{2}^{2}(\rho, \phi) \coloneq \inf_{\Gamma \in \Pi(\rho, \phi)} \int_{X \times X} W_{2}^{2}(\mu, \nu)d \Gamma(\mu,\nu),
	\end{equation}
	where 
	\begin{equation}
		\Pi(\rho, \phi) \coloneq \left \{ \Gamma \in \mathcal{P}\left( X \times X \right) : (\pi_{1})_{\#}\Gamma = \rho,~(\pi_{2})_{\#}\Gamma = \phi\right \}
	\end{equation}
	is the set of couplings between $ \rho $ and $ \phi $. We denote the law of particles by $ \rho \in \PPw $ and make the following assumptions:\\
	\textbf{Assumption 1.} The barycenter $ \bar{\mu} \in X $ of $ \rho$ exists and is the global minimizer of $ F_{\rho}(\nu) : X \to \R $, given by
	\begin{equation}
		F_{\rho}(\nu) = \frac{1}{2} \int_{X} W_{2}^{2}\left( \mu, \nu \right)d \rho(\mu).  
	\end{equation}\\  
	\textbf{Assumption 2.} Each measure $\mu, \bar{\mu} \in X$ is absolutely continuous with respect to the Lebesgue measure.

	For each $ \mu \in X $ we define the measurable map $\Psi_{\Delta t}:  X \to X$ 
	\begin{equation}
		\Psi_{\Delta t}(\mu) \coloneq \left( \left( 1 - \Delta t \right)\text{Id} + \Delta t T\right)_{\#} \mu,
	\end{equation}
	with 
	\begin{equation}
		\bar{\mu} = T_{\#}\mu, 
	\end{equation}
	such that the particle update in Eq.~\eqref{eq:mean_field_particle_update} reads $ 
	\mu_{n}^{k+1} = \Psi_{\Delta t}(\mu_{n}^{k})$. The law of particles $ \rho$ evolves as
	\begin{equation}
		\rho^{k+1} = \left( \Psi_{\Delta t} \right)_{\#} \rho^k. 
	\end{equation} 
	This is the time-discrete mean-field limit, which provides an analytic framework to investigate the optimal transport CBO scheme:

	\begin{proposition}\label{prop:consensus}
		Let the variance of the measure $ \rho^k $ be given by
		\begin{equation}
			\mathcal{V}(\rho^k) = \frac{1}{2}\int_{X} W_{2}^{2}(\mu, \bar{\mu}^{k})d\rho^k(\mu).
		\end{equation}
		For a time step $ \Delta t \in (0,1] $, it holds that
		\begin{equation}
			\mathcal{V}(\rho^{k+1}) \leq \left( 1 - \Delta t \right)^{2} \mathcal{V}(\rho^{k}) .
		\end{equation}
	\end{proposition}

	\begin{proof}
		We start by observing that
		\begin{equation}\label{proof3_eq1}
			\mathcal{V}(\rho^{k+1}) = F_{\rho^{k+1}}(\bar{\mu}^{k+1}) \leq F_{\rho^{k+1}}(\bar{\mu}^{k}) = \frac{1}{2}\int_{X}W_{2}^{2}\left( \mu, \bar{\mu}^{k} \right) d\rho^{k+1}(\mu),
 		\end{equation}
		since $ \bar{\mu}^{k+1}$ is the global minimizer of $ F_{\rho^{k+1}}(\nu) $. Next, we apply the definition of the pushforward on Eq.~\eqref{proof3_eq1}. This yields
		\begin{equation}\label{proof3_eq2}
		\mathcal{V}(\rho^{k+1}) \leq \frac{1}{2}\int_{X}W_{2}^{2}\left( \Psi_{\Delta t}(\mu), \bar{\mu}^{k} \right) d\rho^{k}(\mu).
		\end{equation}
		The map $ \Psi_{\Delta t} $ defines a constant-speed Wasserstein geodesic between $ \mu $ and $ \bar{\mu} $ (Theorem 7.2.2 in~\cite{ambrosioGradientFlows2008}), and therefore satisfies
		\begin{equation}\label{proof3_eq3}
			W_{2}^{2}\left( \Psi_{\Delta t}(\mu), \bar{\mu}^{k} \right) = \left( 1 - \Delta t \right)^{2} W_{2}^{2}\left(\mu, \bar{\mu}^{k} \right).
		\end{equation}   
		Combining \eqref{proof3_eq2} and \eqref{proof3_eq3}, we obtain
		\begin{equation}
			\mathcal{V}(\rho^{k+1}) \leq \frac{1}{2}\left( 1 - \Delta t \right)^{2} \int_{X}W_{2}^{2}\left( \mu, \bar{\mu}^{k} \right) d\rho^{k}(\mu) = \left( 1 - \Delta t \right)^{2}\mathcal{V}(\rho^{k}).
		\end{equation}
	\end{proof}

		
	Proposition~\ref{prop:consensus} demonstrates that the optimal transport formulation of CBO achieves consensus: the variance $\mathcal{V}(\rho^k)$ decays geometrically to zero at a rate of $(1-\Delta t)^{2k}$.

	Variance decay is \emph{necessary but not sufficient} for convergence to a global
	optimizer. Proposition~\ref{prop:consensus} shows the ensemble collapses, but not that
	$\bar{\mu}$ is globally optimal. In classical CBO this link is provided by the Laplace
	principle and a regularity assumption on the objective function; an analogous argument for the OT formulation remains an
	open problem.

	\subsection{Example 4: Square Approximation}\label{sec:mean-field_numerics}
	In the fourth experiment, we investigate whether the optimal transport dynamics proposed in \eqref{sec4:ot_dynamics} are capable of training arbitrarily wide neural networks. To this end, we focus on the empirical risk during training for neural networks with different values of $ M $ and $N$.

	In the implementation of the optimal transport CBO formulation, we consider empirical measures of the form:
	\begin{equation}\label{eq:empirical_measures_experiment}
		\hmu^{k} = \frac{1}{M} \sum_{i=1}^{M} \delta(\bm{w} - \bm{w}_{n,i}^{k})\delta(b - b_{n,i}^{k})\delta(c - c_{n,i}^{k}).
	\end{equation} 
	The dynamics in \eqref{sec4:ot_dynamics}, for $ \mu_n^k $ equal to the empirical measure \eqref{eq:empirical_measures_experiment}, takes the form of a noise-free CBO scheme. Written explicitly for $ (\bm{w}, b,c) $, this gives 	
	\begin{align}
		\bm{w}^{k+1}_{n,i} &= \bm{w}_{n,i}^{k} - \Delta t \left(\bm{w}^{k}_{n, i} - \overline{\bm{w}}^{k}_{\pi^{*}_{n}(i)}\right)\\
		b^{k+1}_{n,i} &= b_{n,i}^{k} - \Delta t \left(b^{k}_{n, i} - \overline{b}^{k}_{\pi^{*}_{n}(i)}\right)\\ 
		c^{k+1}_{n,i} &= c_{n,i}^{k} - \Delta t \left(c^{k}_{n, i} - \overline{c}^{k}_{\pi^{*}_{n}(i)}\right),
	\end{align}
	where $ \pi^*_{n} $ denotes the permutation matrix that represents the optimal coupling of the $ i $-th weight to the barycenter. To facilitate the training, we include a drift parameter $ \lambda $ and add artificial noise to the parameter updates. The complete dynamics are as follows:
		\begin{align}
		\bm{w}^{k+1}_{n,i} &= \bm{w}_{n,i}^{k} - \lambda \Delta t \left(\bm{w}^{k}_{n, i} - \overline{\bm{w}}^{k}_{\pi^{*}_{n}(i)}\right) + \tilde{\sigma}^{k}\sqrt{\Delta t}\bm{\xi}_{n,i}^{k}\label{eq:OT_discrete_w} \\
		b^{k+1}_{n,i} &= b_{n,i}^{k} - \lambda \Delta t \left(b^{k}_{n, i} - \overline{b}^{k}_{\pi^{*}_{n}(i)}\right) + \tilde{\sigma}^{k}\sqrt{\Delta t}\xi_{n,i}^{k} \label{eq:OT_discrete_b} \\ 
		c^{k+1}_{n,i} &= c_{n,i}^{k} - \lambda \Delta t \left(c^{k}_{n, i} - \overline{c}^{k}_{\pi^{*}_{n}(i)}\right) + \tilde{\sigma}^{k}\sqrt{\Delta t}\xi_{n,i}^{k}, \label{eq:OT_discrete_c}
	\end{align}
	with $ \lambda, \tilde{\sigma} > 0$ and $\bm{\xi}_{n,i}^{k} \sim \mathcal{N}\left(\bm{0}, \bm{\mathrm{I}} \right) $. However, the noise lacks a multiplicative term that decreases as the particles converge to the barycenter, similar to the anisotropic noise in \eqref{cbo:sde}. Therefore, we manually reduce $ \tilde{\sigma}^{k} $ according to a predefined schedule.

	We consider a one-dimensional regression problem, where the goal is to approximate the function $ x \mapsto x^2 $ on the domain $ [0,1]$ using a finite two-layer neural network. The neural network has a RELU activation function and we consider various widths $ M $. We sample 5000 data points $ x_s $ uniformly on the interval $ [0,1] $ and apply the function $ y_s = x^2_s + 0.01 \xi_s $, with $ \xi_s \sim \mathcal{N}(0,1)$, to generate the training dataset $\left\{\left( x_{s}, y_{s} \right)\right\}_{s=1}^{5000}$. We divide the training dataset into minibatches of size $ S^\prime$.  We take the squared error loss as the loss function and train the neural networks for 600 epochs. 
	
	The CBO parameters in the experiment are chosen as follows:
	\begin{align*}
		\Delta t = 10^{-1}, \quad \alpha = 10^4, \quad \lambda = 1, \quad \tilde{\sigma} = \sqrt{1.2}, \quad S^\prime = 2500. 
	\end{align*}
	For each particle, we choose a uniform distribution as the initial measure and sample the atoms from $ \bm{w}_{n,i}, b_{n,i}, c_{n,i} \sim \mathcal{U}[-2,2] $. Every 100 epochs, we reduce the noise parameter by a factor of 0.9 every 100 iterations and multiply $ \alpha $ by 10 until it reaches $ 10^7 $. Finally, for each value of $ M $ and $N$, we perform 10 independent simulations and report the average. 

	Figure~\ref{fig:mean_field_M} shows the empirical risk at epoch 600 as a function
	of the network width $M$, with the number of particles fixed at $N = 200$.
	Figure~\ref{fig:mean_field_N} shows the corresponding dependence on $N$, with the
	width fixed at $M = 200$. In both cases the empirical risk decreases monotonically and appears to saturate.

	These experiments confirm that the optimal transport CBO formulation is capable of
	training neural networks. The monotone decrease in $M$ is consistent with the
	infinite-width mean-field model of the neural network, while the monotone decrease
	in $N$ corroborates the mean-field model of CBO. The saturation of both curves indicates that $M = N = 200$ is already close to the respective mean-field regimes. However, computing the barycenter each iteration has a higher computational cost than the classical consensus point, limiting the practical usage of the optimal transport formulation.

	\begin{minipage}[t]{0.5\linewidth}
		\centering
		\includegraphics[width=\linewidth]{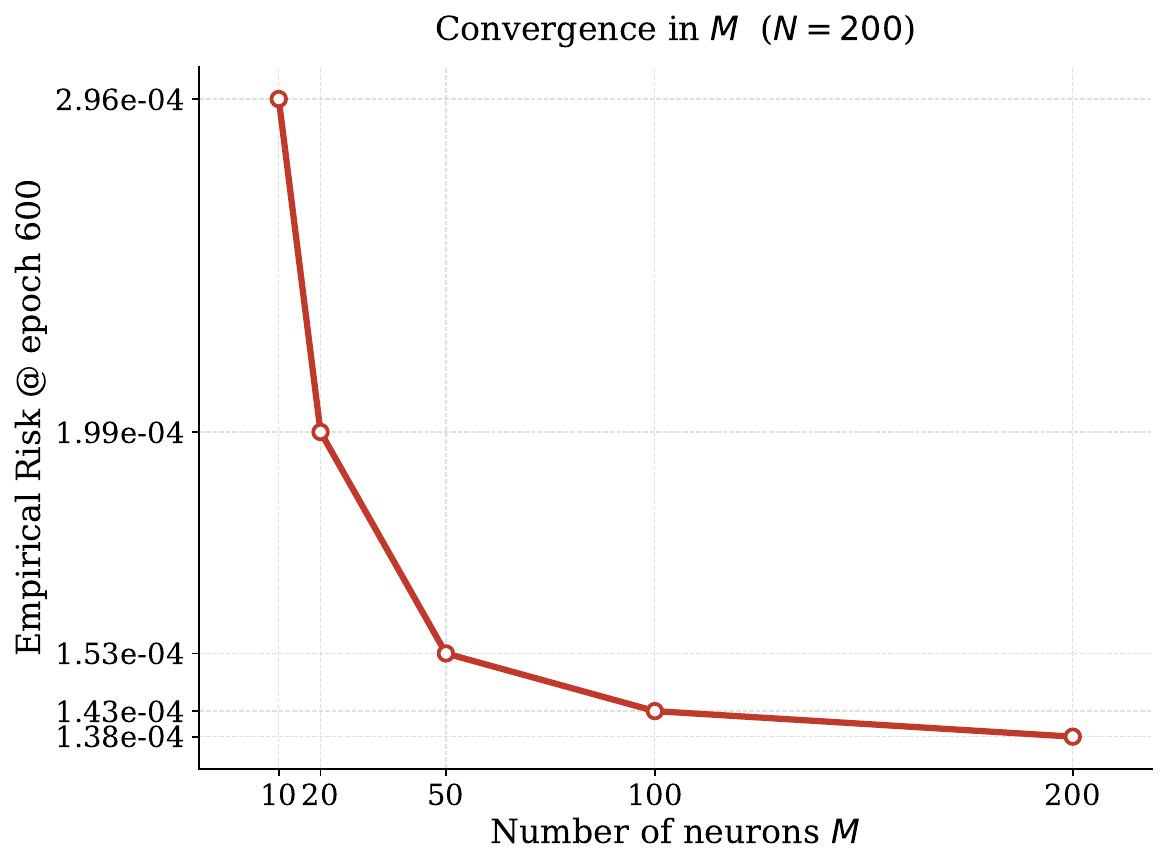}
		\captionsetup{justification=centering}
		\captionof{figure}{Empirical risk $\hat{R}(\mu)$ at epoch 600 as a function of the network width $M$,
		with $N = 200$ particles fixed. Each neural network is represented by a measure and
		trained with the optimal transport formulation of CBO. Results are averaged over
		10 independent runs.}
		\label{fig:mean_field_M}
	\end{minipage}
	\hfill
	\begin{minipage}[t]{0.5\linewidth}
		\centering
		\includegraphics[width= \linewidth]{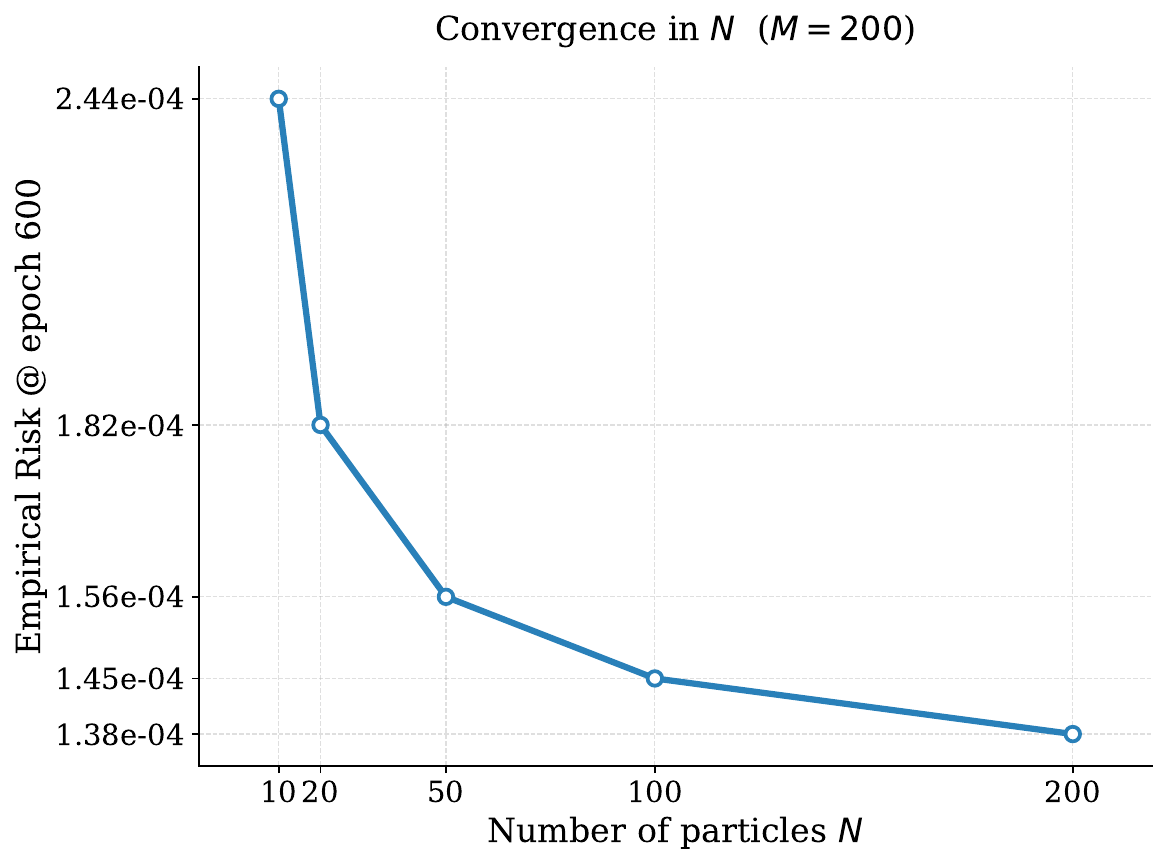}
		\captionsetup{justification=centering}
		\captionof{figure}{Empirical risk $\hat{R}(\mu)$ at epoch 600 as a function of the number of
		particles $N$, with network width $M = 200$ fixed. Each neural network is
		represented by a measure and trained with the optimal transport formulation of CBO.
		Results are averaged over 10 independent runs.}
		\label{fig:mean_field_N}
	\end{minipage}

	\section{Conclusion}\label{section6}
	This work investigated training two-layer neural networks with Consensus-Based Optimization and analyzed its mean-field models. On a smooth regression task, CBO achieved competitive final empirical risk values compared to Adam, while the hybrid method on MNIST classification demonstrated greater robustness and faster convergence than Adam. We hypothesize that the CBO method achieves faster convergence for highly non-convex risk functions, whereas in cases of slow convergence, incorporating local gradient information can be beneficial. In the multi-objective optimization setting, Multi-Task CBO achieves high accuracy with minimal memory overhead by recycling particles across all tasks. 
	
	On the theoretical side, we reformulated the CBO dynamics on $ \R^{M(d+2)} $ within the Wasserstein space $ \Pw $, thereby enabling the training of continuous neural networks with a particle-based method. However, the optimal transport dynamics are currently deterministic, as they do not include Brownian motion. In practice, artificial noise can always be added when considering empirical measures. We presented a time-discrete mean-field formulation over both the neurons ($M \to \infty$) and the particles ($N \to \infty$) and proved that the optimal transport scheme achieves variance decay (Proposition~\ref{prop:consensus}). We emphasize that variance decay establishes consensus among particles but does not, by itself, imply convergence to a global minimizer of the risk; connecting these two notions in the OT setting is an important open problem. Lastly, we illustrated numerically that the empirical risk decreases monotonically in $M$ and in $N$ individually, in agreement with the respective mean-field models.

	\paragraph{Limitations and future work.} The current framework is restricted to two-layer neural networks. Extending the mean-field analysis to deeper architectures ($k$-layer networks, $k > 2$) is a natural direction for future work: this requires a measure-based representation of deep networks and corresponding Barron-type approximation theory, for which some first steps are available in~\cite{eBarronSpaceFlowInduced2022}. On the optimization side, incorporating a rigorous diffusion term into the OT dynamics, or establishing a rigorous proof of convergence, would significantly strengthen the theoretical foundations of the proposed approach.

	\section*{Acknowledgements}
	The work of W.D.D. is supported by the European Union’s Horizon Europe research and innovation program under the Marie Sklodowska-Curie Doctoral Network Datahyking (Grant No. 101072546). The authors thank the Deutsche Forschungsgemeinschaft (DFG, German Research Foundation) for the financial support through 442047500/SFB1481 within the projects B04 (Sparsity fördernde Muster in kinetischen Hierarchien), B05 (Sparsifizierung zeitabhängiger Netzwerkflußprobleme mittels diskreter Optimierung) and B06 (Kinetische Theorie trifft algebraische Systemtheorie). The authors thank the Deutsche Forschungsgemeinschaft (DFG, German Research Foundation) for the financial support HE5386/33-1 Control of Interacting Particle Systems, and Their Mean-Field, and Fluid-Dynamic Limits (560288187), and  HE5386/34-1 Partikelmethoden für unendlich dimensionale Optimierung ( 561130572). This research was partially funded by the KU Leuven Research Fund under grant C14/23/098.


	\bibliographystyle{abbrv}
	\bibliography{references}

\end{document}